\documentclass[12pt]{iopart}

\usepackage{algorithm}
\usepackage{algorithmic}
\usepackage{graphicx}
\let\equation*\relax
\let\endEquation*\relax
\usepackage{amsthm}
\newtheorem{dfn}{Definition}

\newtheorem{theorem}{Theorem}
\newtheorem{lem}{Lemma}
\newcommand{\argmin}{\mathop{\rm arg~min}\limits}

\usepackage{xcolor}

\def\ci{\noexpand{\perp \!\!\! \perp}}
\def\notci{\noexpand{\not \! \perp \!\!\! \perp}}

\begin{document}

\title{Practically Effective Adjustment Variable Selection in Causal Inference}

\author{Atsushi Noda$^1$\footnote{Author to whom any correspondence should be addressed.} and Takashi Isozaki$^2$}
\address{$^1$ Sony Corporation of America, Los Angeles, CA, United States of America}
\address{$^2$ Sony Computer Science Laboratories, Inc., Tokyo, Japan}
\eads{\mailto{Atsushi.A.Noda@sony.com}, \mailto{isozaki@csl.sony.co.jp}}

\begin{abstract}
In the estimation of causal effects, one common method for removing the influence of confounders is to adjust the variables that satisfy the back-door criterion. However, it is not always possible to uniquely determine sets of such variables. Moreover, real-world data is almost always limited, which means it may be insufficient for statistical estimation. Therefore, we propose criteria for selecting variables from a list of candidate adjustment variables along with an algorithm to prevent accuracy degradation in causal effect estimation. We initially focus on directed acyclic graphs (DAGs) and then outlines specific steps for applying this method to completed partially directed acyclic graphs (CPDAGs). We also present and prove a theorem on causal effect computation possibility in CPDAGs. Finally, we demonstrate the practical utility of our method using both existing and artificial data.
\end{abstract}

\vspace{2pc}

\section{Introduction}
Machine learning techniques such as deep learning have shown great proficiency at prediction or classification that outputs values or labels based on multiple input variables. Applications such as recommending new products that consumers are likely to buy based on their purchase logs and identifying defective products based on quality inspection data are now widely utilized throughout the world. At the same time, when considering the effects of government economic policies, medical drug treatments, business marketing strategies, or the improvement of quality defects in manufacturing, it has become crucial to estimate causal effects and to understand what would happen if a particular action were taken. However, general predictive analysis does not consider causal relationships between variables when learning the mapping from inputs to outputs, making it impossible to perform such causal effect calculations. 
In order to accurately capture such causal relationships and complex interactions, a theoretical perspective on certain networks and complex systems is necessary.
To this end, starting with Wright's path diagram \cite{wright1934method}, various theories and computational algorithms have been developed for structural causal models (SCMs) \cite{pearl1995, spirtes2000causation}, which explicitly express causal relationships among variables by means of directed acyclic graphs (DAGs) known from graph theory.
It is also possible to calculate causal effects by adjusting for variables called covariates in SCMs.

Two well-known criteria for adjustment variable selection are the back-door criterion \cite{pearl2009causality} and the adjustment criterion \cite{shpitser2010}. With these criteria, it is necessary to find the variables that satisfy them. When given a set of variables, there is an algorithm consisting of six steps to modify the graph and see if the intervention and outcome variables are completely separated, and to determine whether or not the set of variables is suitable for use as the adjustment variables \cite{shrier2008reducing}. However, this algorithm cannot find the adjustment variable or determine the optimal one. For finding a better set of adjustment variables, there is a criteria for comparing which of two variable sets is preferable in the case of a causal linear model \cite{kuroki2003covariate, Kuroki2004SelectionOI}, as well as a method to find the optimal variable set with the smallest variance \cite{henckel2022graphical, witte2020efficient}. However, since these methods are based on numerical variables and require the assumption of a multivariate Gaussian distribution, we need to cover other situations to deal with real-world data. Elsewhere, a method was proposed to compress confounding variables into a low-dimensional representation using kernels when the intervention variables are binary \cite{cheng2022sufficient}. As is clear from existing studies, there can be multiple sets of variables that satisfy the back-door or adjustment criteria. If there is a sufficient amount of data, the intervention calculation will be the same regardless of which variables that satisfy these criteria are used for adjustment, but if there is not a sufficient amount of data, the accuracy may degrade depending on which variables are actually used.

In this paper, we propose a method to find the optimal adjustment variable by examining not only the graph structure but also the actual data in order to reduce the accuracy degradation of the intervention calculation caused by the lack of data. The recent studies mentioned above \cite{henckel2022graphical, witte2020efficient} have focused on numerical variables, with little discussion of categorical variables. In contrast, we examine multi-valued categorical variables and do not limit the relationship between variables to a linear one. Our algorithm provides a specific procedure for selecting adjustment variables in a simple problem setting where the graph structure is a DAG and the intervention variable is a single variable. We then extend the problem setting and discuss the applicability of the proposed method to the case of completed partially directed acyclic graphs (CPDAGs) or multiple intervention variables.
In particular, for CPDAGs, we present a theorem about conditions and methods for intervention calculations to be possible and then prove the theorem.

\section{Preliminaries}
In this section, we present a model for graphically representing causal relationships. First, we describe nodes and edges as elements of graphs and their associated notations. Then, we explain how to compute causal interventions based on graph structures and the back-door criterion that is important for the computation.

\subsection{Basics of graphical representation}
A graph structure is represented by nodes and edges. Let $V$ denote the set of nodes ($i=1,2,\cdots,n$), and $E$ denote the set of edges $(i, j)$ connecting the nodes, and let the graph $G$ consist of the pair $G=(V,E)$ of nodes $V$ and edges $E$. Two nodes are adjacent if there is an edge between them, and $i$ is the parent of $j$ and $j$ is the child of $i$ if there is a directed edge from $i$ to $j$. If node $i$ is a parent of $j$ and $j$ is a parent of $k$, we say that there is a directed path from $i$ to $k$. If node $i$ can be reached by a directed path from $i$ to $k$, we say that $i$ is an ancestor of $k$ and $k$ is a descendant of $i$. The set of parents of node $i$ is denoted by $pa(i)$, the set of children by $chi(i)$, the set of ancestors by $anc(i)$, the set of descendants by $des(i)$, and the set of all nodes adjacent to node $i$ by $adj(i)$. Assume that the values of all nodes $\mathbf{D} = \{ D_{1}, D_{2}, \cdots , D_{n} \}$ are observed and all values are categorical.

We assume that the DAGs represent the causal relationships among nodes and satisfy the Causal Markov Condition, described by 
\begin{eqnarray}
 X \ci V \setminus \{X, des(X)\} \mid pa(X)
 \label{cmc}
\end{eqnarray}
\cite{spirtes2000causation}, and that all confounding variables necessary for the intervention calculations are observed. The DAG may be given by human knowledge, but in the case where it is estimated from observed data by statistical causal inference \cite{spirtes2000causation}, it is not always uniquely obtained. In other words, we may have Markov-equivalent DAGs, which are expressed in terms of a completed partially directed acyclic graph (CPDAG) \cite{meekUAI95, andersson1997characterization, chickering2002learning}. For such realistic problems, we extend our algorithm for CPDAGs after first describing it based on DAGs. We do not consider maximal ancestral graphs (MAGs) or partial ancestral graphs (PAGs) in this paper.

\subsection{Intervention and adjustment variables}
To determine the causal effect of the intervention variable $X$ on the outcome variable $Y$, we can utilize causal information between variables represented by DAG. The intervention effect from $X$ to $Y$ is expressed by 
\begin{eqnarray}
 P(Y=y|do(X=x))= \sum_{z}P(Y=y|X=x, Z=z)P(Z=z),
\label{interv}
\end{eqnarray}
where $do(X=x)$ is a notation that represents an intervention with $X$ set to $x$ \cite{pearl2009causality}. $Z$ denotes the adjustment variables, and if $Z$ were numerical variables, the integral form would be used instead of the summation form. The mean causal effect can be calculated as the difference between these intervention effects. For example, if $X$ is the presence or absence of a medicine (0 = not administered, 1 = administered), we can determine the effectiveness of the medication from the difference between $do(X = 0)$ and $do(X = 1)$. In the context of causal effect calculations, $X$ and $Y$ are often assumed to be binary, but in this paper, we assume that they are multi-valued. Also, while we assume that both $X$ and $Y$ are single variables, we will later discuss the case of multiple intervention variables (joint intervention).

The back-door criterion \cite{pearl2009causality} can be used for the selection of adjustment variable $Z$. It is defined as follows.
\begin{dfn}[Back-door criterion, Pearl, 2009]
In a DAG $G$, a set of variables Z satisfies the back-door criterion for $\{X,Y\}$ if
\begin{enumerate}
 \item There is no directed path from $X$ to any element of $Z$. 
 \item In a graph of $G$ excluding edges out of $X$, $X$ and $Y$ are $d$-separated by $Z$.
\end{enumerate}
\end{dfn}
The important point here is that there is always a set of variables in each $G$ that satisfies the back-door criterion. For example, $pa(X)$, $anc(X)$, and any of the variables of the non-ancestors of $X$ added to these sets also satisfy the backdoor criterion. Therefore, multiple sets of variables generally satisfy the back-door criterion. If $X$ and $Y$ are $d$-separated from the beginning (e.g., when there is no parent of $X$), then the empty set $\phi$ satisfies the back-door criterion.

Techniques other than the back-door criterion for intervention calculation include the front-door criterion \cite{pearl1995} and the instrumental variable \cite{angrist1996identification} methods. The advantages and disadvantages of each method vary depending on conditions such as the observation status of variables and the shape of the DAG, but we focus only on the back-door criterion because of its generality. There is also an extension of the back-door criterion called the adjustment criterion \cite{shpitser2010}, but since the smallest set of variables that satisfies the back-door criterion also satisfies the adjustment criterion, we focus only on the back-door criterion for simplicity.

\section{Methodology}
As stated earlier, if variables that satisfy the back-door criterion are found, the intervention effect can be calculated by adjusting for these variables to eliminate the bias caused by confounding factors. However, when calculating intervention effects, using an unnecessarily large number of variables for adjustment simply because they satisfy the back-door criterion may lead to an accuracy degradation, depending on the amount of data. Therefore, in this paper, we propose a new criterion for the selection of adjustment variables and show the specific procedure.

\subsection{Problems in handling real data}
One problem when calculating intervention effects in real-world analyses is that there are multiple sets of variables that satisfy the back-door criterion, and it is difficult to know which one should be used for adjustment. In any case, when a causal relationship between variables can be constructed based on expert knowledge, or when the DAG is estimated from the observational data, it is necessary to find a set of variables that satisfy the back-door criterion. As mentioned earlier, $pa(X)$ certainly satisfies the back-door criterion, and the union of $pa(X)$ and \textit{any variable that is not a descendant of $X$} also satisfies the criterion. Furthermore, using \textit{all variables that are not descendants of $X$} also satisfies the back-door criterion, but using an unnecessarily large number of variables for adjustment increases the probability of accuracy degradation of the intervention calculations. Since covariate adjustment in causal inference is based on probability theory, it is important to keep in mind that statistical processing is performed when this is implemented with actual data. For example, if the number of adjustment variables is large or the number of categories for each variable is large, there may be a shortage of data, which can lead to a reduction in accuracy.

Here are two simple examples of degradation of the calculation accuracy when using $pa(X)$ as the adjusting variable. Figure \ref{fig:1} (Left) shows that the correlation between the intervention variable $X$ and its parent $P1$ is strong, which can happen when variables semantically similar to $X$ are included in the data. Consider the following example. If both $X$ and $P1$ are binary variables and $P(X=1|P1=0) \ll P(X=0|P1=0)$, then there will be fewer data combinations where $X=1$ and $P1=0$, making it difficult to obtain statistically accurate results when calculating $do(X = 1)$. In addition, Fig. \ref{fig:1} (Right) shows an example where the intervention variable $X$ has three parents $\{P1, P2, P3\}$, but only $V1$ is sufficient to close the back-door path. These two examples demonstrate that it may not be appropriate to choose $pa(X)$ as the adjustment variable if the only reason is that it always satisfies the back-door criterion.

\begin{figure}[t]
\centering
\includegraphics[scale=0.5]{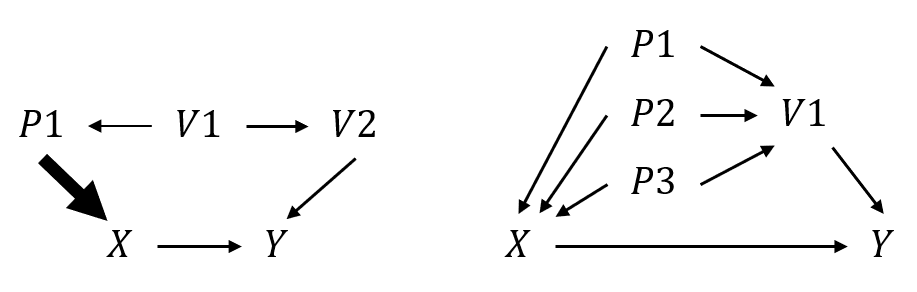}
\caption{(Left) There is a strong correlation between $X$ and $P1$, indicated by the bold arrow. $P1$ definitely closes the back-door path, but so do $V1$ and $V2$. (Right) $X$ has three parents $\{P1, P2, P3\}$, but only $V1$ is sufficient to close the back-door path.}
\label{fig:1}
\end{figure}

Next, we explain the situation where accurate results cannot be obtained statistically, along with specific numerical values. In the DAG shown in Fig. \ref{fig:2}, the causal effect from $X$ to $Y$ can be calculated with $Z$ as the adjustment variable. In the conditional probability table for $Y$, the numbers in parentheses indicate the number of data samples, and the numbers above them indicate the probability value of $Y$ obtained by aggregating the data. However, as we can see from the table, there is almost no data where $X=0$ and $Z=0$. Therefore, in the calculation of $do(X=0)$ expressed in 
\begin{eqnarray}
 P(Y=y|do(X=0)) &=& P(Y=y|X=0,Z=0)P(Z=0) \nonumber \\
 & & + P(Y=y|X=0, Z=1)P(Z=1),
 \label{interv_ex}
\end{eqnarray}
the value of the first term $P(Y=y|X=0,Z=0)$ is calculated at the rate of $P(Z=0)=101/301 \simeq 33\%$ because the value of $X$ is fixed during the intervention calculation regardless of the value of $Z$. 

If the true probability value of $P(Y=1|X=0,Z=0)$ is close to $0.6 \sim 0.8$, the same as for the other conditions, it is clearly a negative effect to have $P(Y=1|X=0,Z=0)=0$ due to only one data satisfying this condition, as shown in Fig. \ref{fig:2}. Thus, even if variables that satisfy the back-door criterion are used in the calculation, the accuracy of the intervention calculation is unstable due to the lack of data. This is why it is important to select the best adjustment variables by considering the data, not only the structural information of the DAG.

\begin{figure}[t]
\centering
\includegraphics[scale=0.6]{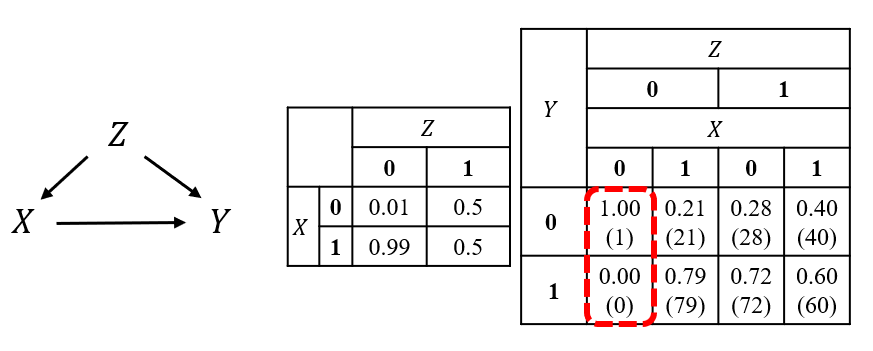}
\caption{Specific DAG and CPT examples that reduce accuracy. In the conditional probability table for $Y$, the numbers in parentheses indicate the number of data samples, and the numbers above them indicate the probability value of $Y$ obtained by aggregating the data. In this example, the computational accuracy of $do(X=0)$ calculus will be unstable due to the small number of data samples available for the cases $X=0$ and $Z=0$. It indicates that it is important to examine the data as well as the DAG structure.}
\label{fig:2}
\end{figure}

\subsection{Proposed selection method}
Here, we propose an adjusted variable selection method that mitigate the above problem. The algorithm consists of two main steps. In Step 1, the minimal set of variables that satisfy the back-door criterion for $\{X,Y\}$ is enumerated, and then in Step 2, the optimal set of variables is selected from these sets.

In Step 1, if there are two sets $A$ and $B$ that satisfy the back-door criterion for $\{X,Y\}$, and if $B \subseteq A$, then it is better to choose B as the adjustment variable in order to avoid the above problem of insufficient data. In this case, the minimal variable set that satisfies the back-door criterion is also the minimal variable set that satisfies the adjustment criterion \cite{shpitser2010}. It is also the minimal adjustment variable for estimating the causal effect from $X$ to $Y$, as shown in Theorem 4.4 \textit{Minimal Covariate Adjustment} in \cite{textor2012}. In Step 1-1, edges adjacent to nodes other than $X$, $Y$, $anc(X)$, or $anc(Y)$ on $G$ are deleted. This is the process for selecting the adjustment variable from the ancestors of $X$ or $Y$. In Step 1-2, edges outgoing from $X$ on $G$ are deleted. This is because the adjustment variable must not be a descendant of $X$, as is clear from the first of the back-door criteria. For Steps 1-1 and 1-2, in Definition 4.3 (Moral Graph Criterion) of \cite{textor2012}, we confirmed that the adjustment variable $Z\subseteq anc(X\cup Y)\setminus des(X)$ satisfies the moral graph criterion, and the minimal set of variables satisfying the moral graph is the same as the minimal set satisfying the backdoor criterion. In Step 1-3, all paths from $X$ to $Y$ on $G$ are enumerated and denoted as $Path_{XY}$. The paths can include head-to-head ones. In Step 1-4, search for variable sets that $d$-separate all paths in $Path_{XY}$. If there is an inclusion relation between the variable sets, the superset is eliminated from the candidates (i.e., only the minimal set is retained). This procedure results in minimal variable sets $\mathbf{S}= \{ \mathbf{S}_{1}, \mathbf{S}_{2}, \cdots , \mathbf{S}_{n} \}$ that satisfy the back-door criterion for $\{X,Y\}$.

In Step 2, the optimal adjustment variable is selected from the $\mathbf{S}$ obtained in Step 1. A simple method would be to randomly sample the data and repeatedly perform intervention calculations using all of $\mathbf{S}$ for each adjustment to see which one has the smallest error, but this is computationally expensive and impractical. As an alternative, we propose a method to select the variable with the smallest correlation with $X$. Here, we assume that all variables are categorical, so we use the mutual information to calculate the correlation, 
 \begin{eqnarray}
 Cor(X, \mathbf{S}_i) &=& I(X, \mathbf{S}_i) \nonumber \\
 &=& \sum_{x \in X} \sum_{s \in Comb(\mathbf{S}_i)} P(x, s) 
 \log\left(\frac{P(x, s)}{P(x)P(s)}\right)
 \label{MI}
\end{eqnarray}

where $Comb(\mathbf{S}_i)$ is a function that enumerates all possible combinations of values for a set of variables $\mathbf{S}_i$. For example, if $\mathbf{S}_i=\{A, B\}$ and both are binary variables, then $Comb(\mathbf{S}_i)=\{(a_1,b_1),(a_1,b_2),(a_2,b_1),(a_2,b_2)\}$. The final adjustment variable is then determined to be $\mathbf{S}_i$, which has the minimum amount of mutual information with $X$. The steps above are specified in Algorithm \ref{alg1}. 
We call this algorithm CAVS (Correlation based Adjustment Variable Selection).
Experimental results will be presented later to show that this criterion actually works.

\begin{algorithm}[t]
\caption{CAVS algorithm}  
\label{alg1}    
\begin{algorithmic}
\STATE {\textbf{input}: $G,X,Y, \mathbf{D}$ $G$: DAG, $X$: intervention variable, $Y$: outcome variable, $\mathbf{D}$: data}
\STATE {\textbf{initialize}: $Z=\phi, \mathbf{S}=\phi$}
\STATE {\textbf{Step 1}: Enumerate the minimal set of variables $\mathbf{S}= \{ \mathbf{S}_{1}, \mathbf{S}_{2}, \cdots , \mathbf{S}_{n} \}$ that satisfy the back-door criterion for $\{X,Y\}$.}
\STATE {Step 1-1}: Delete $(v, i)_{i \in adj(v)}$ $\{v \in V$ $\setminus$ \{ $X$, $Y$, $anc(X)$, $anc(Y)$ \} \} from $G$.
\STATE {Step 1-2}: Delete $(X, i)_{i \in chi(X)}$ from $G$.
\STATE {Step 1-3}: $Path_{XY} \leftarrow$ Enumerate paths from $X$ to $Y$ on $G$.
\STATE {Step 1-4}: Search for variable sets that $d$-separate all paths in $Path_{XY}$.
\begin{description}
    \STATE $V_{pathxy} \leftarrow$ Nodes included in $Path_{XY}$
    \FOR{$i = 1$ \textbf{to} $|V_{pathxy}|$}
        \FOR{Variable combinations $\mathbf{v} (|\mathbf{v}|=i)$ in $V_{pathxy}$}
            \IF{$\mathbf{v}$ $d$-separates all paths in $Path_{XY}$ and $\mathbf{v}$ is not a superset of any variable set in $\mathbf{S}$}
                \STATE $\mathbf{S} = \mathbf{S} \cup \mathbf{v}$
            \ENDIF
        \ENDFOR
    \ENDFOR
\end{description}
\STATE {\textbf{Step 2}: Select the set of variables $\mathbf{S}_{i}$ in $\mathbf{S}$ that has the smallest correlation with $X$.}
\STATE $Z=\argmin_{\mathbf{S}_i \in \mathbf{S}} Cor(X,\mathbf{S}_i)$
\STATE {\textbf{return}} $Z$
\end{algorithmic}
\end{algorithm}

Here, we show a specific steps of CAVS using a simple DAG shown in Fig. \ref{fig:alg-step-example}.
In Step 1-1, only $V4$ is not an ancestor of $X$ or $Y$, so delete edges $(V3, V4)$ and $(V5, V4)$.
In Step1-2, delete $(X, V2)$, the only edge that out of $X$.
In Step1-3, $X \leftarrow V3 \leftarrow V6 \rightarrow V8 \rightarrow Y$ and $X \leftarrow V7 \leftarrow V8 \rightarrow Y$ are enumerated as $Path_{XY}$.
In Step1-4, examine the minimal set of variables that $d$-separates $X$ and $Y$ in the combination of nodes $V_{pathxy}=\{V3, V6, V7, V8\}$, and we obtain $\mathbf{S} = \{V3, V7\}\{V6, V7\}\{V8\}$.
Then, in Step 2, we will find the set of variables with the smallest correlation with $X$, but this depends on the data and is omitted from this example.

So far, we have assumed that there is a single intervention variable. However, there are situations where we want to control multiple variables and observe changes in the outcome variable. The admissibility criterion \cite{DBLP:conf/uai/PearlR95} was defined as a condition under which the effects of a joint intervention become identifiable. Therefore, to extend CAVS to joint intervention, it is sufficient to search for a set of variables that satisfy the admissibility criterion rather than the back-door criterion for Step 1. For Step 2, the intervention variable $X$ and the adjustment variable candidate $\mathbf{S}_i$ may be multiple variables, but a detailed study is needed to determine which criteria are appropriate in such cases (a point we leave to future work).

\begin{figure}[t]
\centering
\includegraphics[scale=0.5]{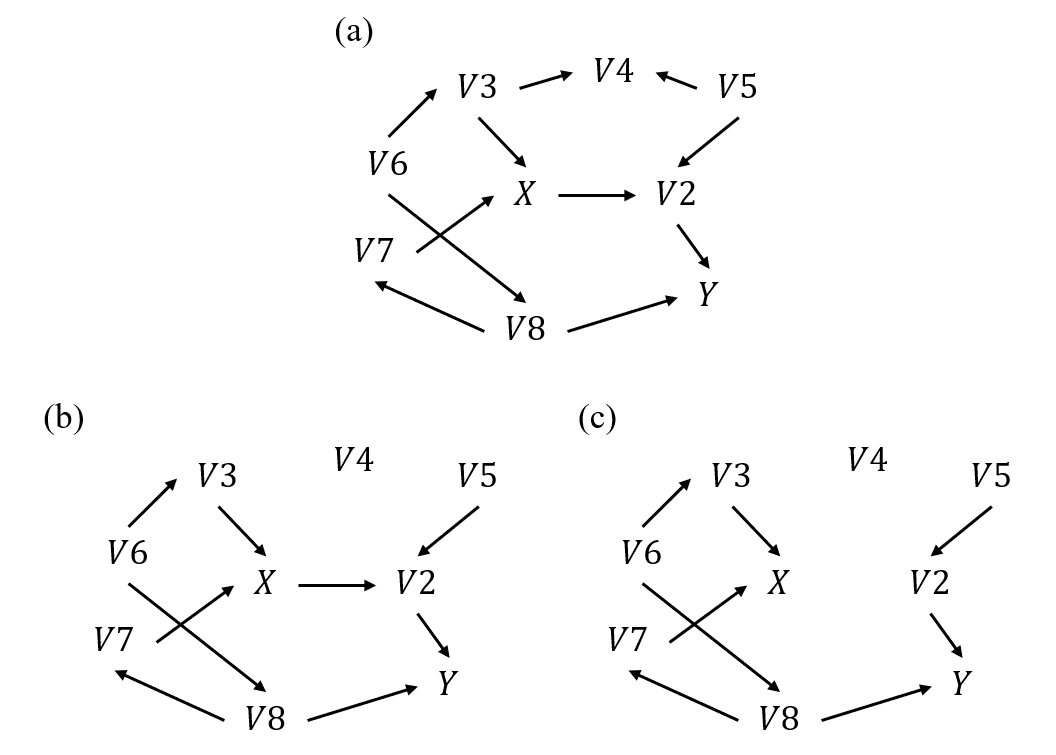}
\caption{An example explaining each step of CAVS algorithm. (a) shows the DAG, (b) shows the graph after Step 1-1 and (c) shows the graph after Step 1-2.}
\label{fig:alg-step-example}
\end{figure}

\subsection{Extension of CAVS to CPDAG}
We assumed DAGs in Algorithm \ref{alg1}, but since it is possible that only CPDAGs can be estimated from observed data, we consider here the case where our method is extended to CPDAGs. 
In this section, we present and prove a theorem that enables the computation of interventions based on CPDAGs, only determining the direction of some edges.
As previously stated, we can use the back-door criterion in the selection of adjustment variables for a DAG. 
However, the adjustment criterion \cite{shpitser2010} is more comprehensive than the back-door criterion defined as follows.

\begin{dfn}[Adjustment Criterion (AC), Shpitser et al., 2010]
In a DAG $G$, a set of variables $Z$ satisfies the adjustment criterion for $\{X,Y\}$ if 
\begin{enumerate}
    \item No element in $Z$ is a descendant in $G$ of any $W \in V \setminus X$ that lies on a causal path from $X$ to $Y$.
    \item All non-causal paths in $G$ from $X$ to $Y$ are blocked by $Z$.
\end{enumerate}
\label{ac}
\end{dfn}
In extending this criterion to CPDAG, we introduce the following concept called amenability \cite{perkovic2015}.

\begin{dfn}[Amenability for DAGs and CPDAGs, Perkovi\'{c} et al., 2015]
A DAG or CPDAG is said to be adjustment amenable, relative to $\{X,Y\}$ if every possibly directed path from $X$ to $Y$ in $G$ starts with a directed edge out of $X$.
\label{ame}
\end{dfn}
According to this concept, the adjustment criteria for DAGs are generalized to include CPDAGs as well, which is defined as the generalized adjustment criterion \cite{perkovic2015} below.

\begin{dfn}[Generalized Adjustment Criterion (GAC), Perkovi\'{c} et al., 2015]
In a DAG or CPDAG $G$, a set of variables $Z$ satisfies the generalized adjustment criterion for $\{X,Y\}$ if 
\begin{enumerate}
    \item $G$ is adjustment amenable relative to $\{X,Y\}$.
    \item No element in $Z$ is a possible descendant in $G$ of any $W \in V \setminus X$ that lies on a causal path from $X$ to $Y$.
    \item All definite status non-causal paths in $G$ from $X$ to $Y$ are blocked by $Z$.
\end{enumerate}
\label{gac}
\end{dfn}
A node has a definite status on a path if it is determined to be a collider or non-collider on that path, and a path has a definite status if all nodes on that path except those at both ends have a definite status.
We now introduce the following theorem \cite{perkovic2015} for DAG and CPDAG.
\begin{theorem}[Perkovi\'{c} et al., 2015]
$Z$ is an adjustment set relative to $\{X,Y\}$ in a DAG or CPDAG $G$ if and only if $Z$ satisfies the generalized adjustment criterion relative to $\{X,Y\}$ in $G$.
\label{theorem}
\end{theorem}
Additionally we show two complementary problems \cite{perkovic2015} below related to the proof of Theorem \ref{theorem} and the new Theorem \ref{theo2}, which will be presented later.
\begin{lem}[Perkovi\'{c} et al., 2015]
\label{lem1}
Let condition 1 of the GAC be satisfied relative to $\{X,Y\}$ in a CPDAG $G$. Then the following two statements are equivalent:
\begin{enumerate}
    \item $Z$ satisfies condition 2 of the GAC relative to $\{X,Y\}$ in $G$.
    \item $Z$ satisfies condition 1 of the AC relative to $\{X,Y\}$ in every DAG in the Markov equivalence class of $G$.
\end{enumerate}
\end{lem}

\begin{lem}[Perkovi\'{c} et al., 2015]
\label{lem2}
Let condition 1 of the GAC be satisfied relative to $\{X,Y\}$ in a CPDAG $G$, and let $Z$ satisfy condition 2 of the GAC relative to $\{X,Y\}$ in $G$. Then the following two statements are equivalent:
\begin{enumerate}
    \item $Z$ satisfies condition 3 of the GAC relative to $\{X,Y\}$ in $G$.
    \item $Z$ satisfies condition 2 of the AC relative to $\{X,Y\}$ in every DAG in the Markov equivalence class of $G$.
\end{enumerate}
\end{lem}
The proofs of Lemmas \ref{lem1}, \ref{lem2} and the proof of Theorem \ref{theorem} using them are provided in the original paper \cite{perkovic2015}.
Theorem \ref{theorem} shows that the satisfaction of GAC is a necessary and sufficient condition for adjustment variables in CPDAGs, but there is no mention of the case where GAC is not satisfied.
Therefore, we present a new theorem on the conditions for satisfying GAC and what can be done when GAC is not satisfied.
\begin{theorem}
In CPDAG, the intervention effect from the intervention variable to the outcome variable is computable if all edges adjacent to the intervention variable are directed edges.
\label{theo2}
\end{theorem}
\begin{proof}
Theorem \ref{theorem} states that only variables that satisfy GAC are adjustment variables in CPDAG.
Therefore, we consider whether variables satisfy GAC by examining two cases, where $X$ is the intervention variable and $Y$ is the outcome variable: 
(1) all edges adjacent to $X$ are directed edges, and  (2) not all edges adjacent to $X$ are directed edges. 

We first examine the case where all edges adjacent to $X$ are directed edges. 
Regarding the adjustment amenable condition, which is one of the conditions of GAC (condition 1), according to Definition \ref{ame}, the CPDAG is adjustment amenable if every possibly  directed path from $X$ to $Y$ starts with a directed edge out of $X$.
Therefore, if $(X, i)_{i \in adj(X)}$ are all directed edges, then its CPDAG is adjustment amenable with respect to $\{X,Y\}$.
Also, according to Lemma \ref{lem1}, if condition 1 of GAC (Definition \ref{gac}) is satisfied for $\{X,Y\}$ in its CPDAG, then the variables satisfying condition 2 of GAC satisfies condition 1 of AC (Definition \ref{ac}) in all Markov-equivalent DAGs.
Also, according to Lemma \ref{lem2}, if conditions 1 and 2 of GAC are satisfied for $\{X,Y\}$ in its CPDAG, then the variables satisfying condition 3 of GAC satisfies condition 2 of AC in all Markov-equivalent DAGs.
That is, if a CPDAG is adjustment amenable with respect to $\{X,Y\}$, then the adjustment variable is common in all Markov-equivalent DAGs.
The fact that the adjustment variables are common across these DAGs means that the results of the intervention calculations using that variables for adjustment are equal. This ensures that the intervention effect calculated from any one of these DAGs is guaranteed to represent the intervention effect in the CPDAG, making the intervention effect in the CPDAG computable.  Therefore, for this case, Theorem 2 is shown to be valid.

Next, we consider the case where not all edges adjacent to $X$ are directed edges. In this case, 
its CPDAG may not be adjustment amenable with respect to $\{X,Y\}$.
Specifically, as in Definition \ref{ame}, the CPDAG is not adjustment amenable if at least one of the possible directed path from $X$ to $Y$ starts with an undirected edge adjacent to $X$, or is adjustment amenable if there are undirected edges adjacent to $X$ but the undirected edges are not related to the possible directed paths from $X$ to $Y$.
If we can determine the direction of at least the edges adjacent to X among all the undirected edges in the CPDAG, then the CPDAG becomes adjustment amenable with respect to $\{X,Y\}$, so that the Markov-equivalent DAGs are limited and there are variables that satisfy the GAC. 
Then, as mentioned before, Lemmas \ref{lem1} and \ref{lem2} make the adjustment variable common in all limited Markov-equivalent DAGs.
As in the previous case, the intervention effect in the CPDAG is computable using any one of these DAGs, and Theorem 2 is shown to be valid for this case as well.
Since Cases (1) and (2) are mutually exclusive, and Theorem 2 is shown to be valid in each case, it follows that Theorem 2 is always valid.
\end{proof}

Figure \ref{fig:4} (a) shows an example of CPDAG. In this case, there is no variable that satisfies GAC because it is not adjustment amenable with respect to $\{X,Y\}$.
Decomposing this CPDAG into Markov-equivalent DAGs, where DAGs (b) and (c) have \{Z1\} or \{Z2\}, DAG (d) has an empty set as the smallest set of variables satisfying AC.
However, this graph becomes adjustment amenable with respect to $\{X,Y\}$ if we determine the direction of the edges between $X-Z1$ among the undirected edges in (a).
If the direction of the edge is $(X,Z1)$, the DAG is restricted to the DAG in (d), and if the direction of the edge is $(Z1,X)$, it is restricted to the DAG in (b) or (c). In both cases, the minimal set of variables satisfying AC can be uniquely determined. Therefore, Step 1 of Algorithm \ref{alg1} should be executed using the DAG in (d) in the former case, and either the DAG in (b) or (c) in the latter case, and Step 2 can be executed as is because it does not depend on the graph structure.
If the intervention variable is not $X$ but $V$ in Fig. \ref{fig:4}, then this CPDAG is adjustment amenable with respect to $\{V,Y\}$, so there is an adjustment variable that satisfies GAC.
In fact, $\{Z2\}$ is the smallest set of variables satisfying AC in any of the DAGs in (b), (c), and (d).
Therefore, even for CPDAGs where no adjustment variables exist, an adjustment variable that satisfies GAC (or the back-door criterion) can exist by determining the direction of edges adjacent to $X$. Then, the optimal adjustment variable can be found by the proposed algorithm.

Using Theorem \ref{theo2}, the procedure for applying the proposed algorithm CAVS to CPDAG is as follows.
If the edges adjacent to the intervention variable are all directed edges, CAVS can be used for any of the Markov-equivalent DAGs. If there are undirected edges adjacent to the intervention variable, the Markov-equivalent DAGs are limited if the direction of those edges can be determined, and we can use CAVS on any of those DAGs.
In other words, for CPDAGs that do not have adjustment variables, the adjustment variables that satisfy the GAC (or the back-door criterion) can be obtained by determining the direction of the edges at least adjacent to $X$, and the optimal adjustment variable can be found by CAVS.
This is practically useful considering the reality that CPDAGs are estimated from the observed data, and it broadens the possibility of intervention calculations in real-world applications.

\begin{figure}[t]
\centering
\includegraphics[scale=0.5]{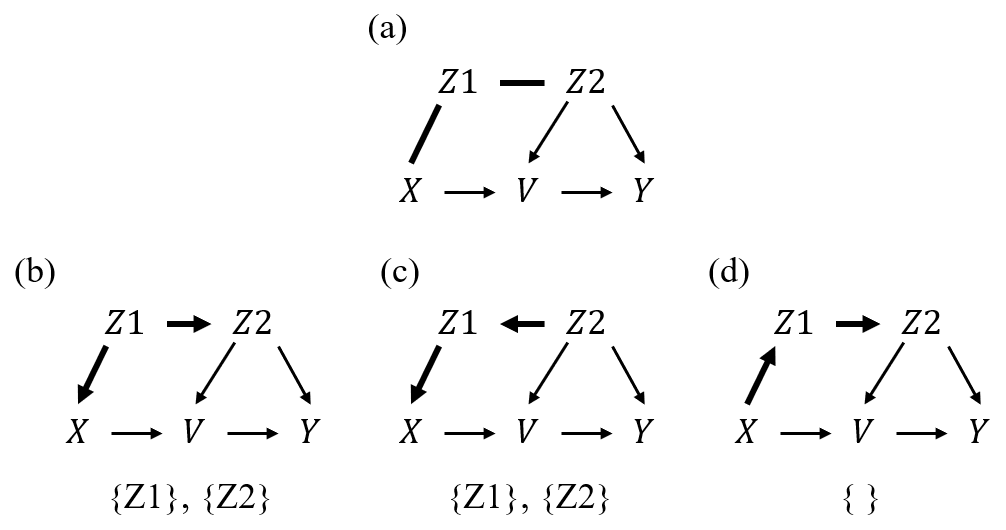}
\caption{(a) An example of CPDAG. (b) (c) (d) Markov-equivalent DAGs, but with different adjustment variables for each graph. (b) and (c) have the same direction of edges adjacent to $X$ and the same adjustment variables $\{Z1, Z2\}$. Otherwise, the direction of the edges adjacent to $X$ in (d) is different from (b) and (c), and the adjustment variable is empty in (d). In other words, the CPDAG in (a) can be divided into \{(b), (c)\} and \{(d)\}, depending on the direction of the edges adjacent to X, and the adjustment variables can be determined.}
\label{fig:4}
\end{figure}

\section{Related Work}
For covariate adjustment, if a DAG is known, we generally focus on whether the adjustment variable satisfies the back-door criterion or the adjustment criterion. However, satisfying the back-door criterion is not always sufficient, and it is known that adjustment by the instrument variable can lead to Z-bias \cite{ding2017instrumental, hernan2024causal}. For example, $pa(X)$ always satisfies the back-door criterion, but if there is an instrument variable in it, it is better to remove it from the adjustment variables. In this paper, this bias can be removed by selecting the smallest set of variables that satisfy the back-door criterion. On the other hand, while practical variable selection methods when the DAG is unknown have been discussed \cite{vanderweele2011new, vanderweele2019principles}, we do not compare them with our study because the assumptions regarding the existence of the DAG are different. In calculating intervention effects, some approaches \cite{rosenbaum1983central} estimate propensity scores and then remove bias through stratification, matching, inverse weighting, and other techniques. However, even in such approaches, the selection of explanatory variables is important when estimating propensity scores. For example, simulations have shown that selecting variables that have a strong relationship with the intervention variable has a negative impact on the calculation of the intervention effect \cite{brookhart2006variable}. Another study compared the results of intervention effect in the case where \textit{all variables are selected}, \textit{variables are selected based on prior knowledge}, and \textit{variables are selected based on the strength of association with the intervention and outcome variables}, etc \cite{patrick2011implications}. These results indicate that it is better not to use variables that are strongly associated with the intervention variable. In this paper, we have summarized how such negative effects occur, and incorporated the findings of these existing studies into the algorithm in order to avoid such scenarios.

In addition, another study showed that adjustment of a variable that is influenced by the intervention variable can introduce bias \cite{rosenbaum1984consequences}, thus suggesting that only variables that are ahead of the intervention variable in time should be chosen. However, it should be noted that even if the variable is ahead of the intervention variable, Z-bias or M-bias \cite{cole2010illustrating, greenland2003quantifying, vanderweele2019principles, pearl2009causality} may occur. Both biases can be avoided by obtaining the DAG, and the structural approach, which constructs the DAG based on knowledge or estimated from data, has an advantage in this regard.

As for CPDAG, one study has proposed an extension of its back-door criterion called the generalized back-door criterion \cite{maathuis2015generalized}. However, although this study mentions that $pa(X)$ satisfies the generalized back-door criterion, it does not mention how to select variables when there are other sets of variables that satisfy this criterion. There is another method that efficiently calculates intervention effects in each Markov-equivalent DAG and then integrates the results \cite{maathuis2009estimating}. As an alternative, we propose an approach to distinguish between cases where the Generalized Adjustment Criterion (GAC) \cite{perkovic2015} is and is not satisfied, and to narrow down DAGs in the latter case by determining the direction of certain edges.

\section{Experiment}
We investigated the accuracy of CAVS using existing datasets and artificially generated data. On the existing datasets, we examine which variables CAVS selects as adjustment variables and how accurately it calculates interventions compared to using other adjustment variables, while specifically showing the positional relationship between the intervention variable, the outcome variable, and the adjustment variables on the DAG. For artificial data, we show that CAVS is superior to some baseline methods on various randomly generated graphs.

We examine the error by measuring the difference between the intervention results when the data are randomly sampled and when sufficient data are available as a reference value. Since there may be zero frequencies, we evaluate the gap between the reference outcome variable $Y$ and the estimated outcome variable $\hat{Y}$ by the cosine distance, as 

\begin{eqnarray}
 \label{error_cosd} 
 Error &=& {\it E_{D_k}} [{\it E_{X}} [Cos(Y, \hat{Y})]] \nonumber \\
 &=& \frac{1}{K} \sum\limits_{k=1}^{K} \frac{1}{J} \sum\limits_{j=1}^{J} 
 \left(1- \frac{\sum\limits_{i=1}^{I} P(Y_{i}^{j}) \cdot P(\hat{Y}_{i}^{j})}
 {\sqrt{\sum\limits_{i=1}^{I} {P(Y_{i}^{j})^2}} \cdot 
 \sqrt{\sum\limits_{i=1}^{I} P(\hat{Y}_{i}^{j})^2}} \right), \nonumber \\
 P(Y_{i}^{j}) &=& P(Y=y_i|do(X=x_j); \mathbf{D}), \nonumber \\
 P(\hat{Y}_{i}^{j}) &=& P(Y=y_i|do(X=x_j); {\mathbf D}_k), \nonumber
\end{eqnarray}

where $I$ is the number of categories in $Y$ and $i$ is their index, $J$ is the number of categories in $X$ and $j$ is their index, $K$ is the number of randomly sampled subsets, and ${\mathbf D}_k$ is the randomly sampled $k$-th subset data from the observed data ${\mathbf D}$.

\subsection{Existing dataset}
We applied CAVS to two existing datasets, Insurance \cite{binder1997adaptive} and Hailfinder \cite{abramson1996hailfinder}, and investigated its performance in detail. Each dataset has a DAG structure and conditional discrete probability distributions. For each of these two datasets, after specifically selecting intervention and outcome variables, we show which variables CAVS actually selected from multiple candidate adjustment variables. We also show where each of the multiple candidate adjustment variables is located on the graph, the mutual information between the adjustment variables and the intervention variable, and in which cases the accuracy degrades.

First, in the Insurance example, Fig. \ref{fig:insurance} shows the overall structure of the DAG, with the node \textit{OtherCarCost} as the outcome variable and \textit{Accident} one level above as the intervention variable. In this case, if we enumerate the candidate adjustment variables that satisfy the back-door criterion described in Step 1 of Algorithm \ref{alg1}, we obtain \{RuggedAuto\}, \{VehicleYear, MakeModel\}, \{SocioEcon, RiskAversion, Antilock\}, \{Age, RiskAversion, Antilock, RiskAversion, Antilock\}, \{RiskAversion, DrivingSkill, Antilock\}, and \{Antilock, DrivQuality\}. \{Mileage\} is a parent of the intervention variable, but since it does not contribute to closing the back-door path, it is not selected as a candidate adjustment variable. Next, we generated 200,000 data samples based on the conditional discrete probability distribution of the dataset, and calculated the mutual information utilizing Eq. (\ref{MI}) as a measure of the correlation between the intervention variable and the candidate adjustment variables. We calculated the intervention results for the outcome variable $Y$ based on the value calculated using all 200,000 samples and estimated the outcome variable $\hat{Y}$ using the subset data ${\mathbf D}_k$ obtained by randomly sampling with 2000 samples. We sampled ${\mathbf D}_k$ five times and calculated the error between $Y$ and $\hat{Y}$ using Eq. (\ref{error_cosd}). The results are listed in Table \ref{tab:table1}.

\begin{figure}[t]
 \centering
 \includegraphics[width=0.84\columnwidth]{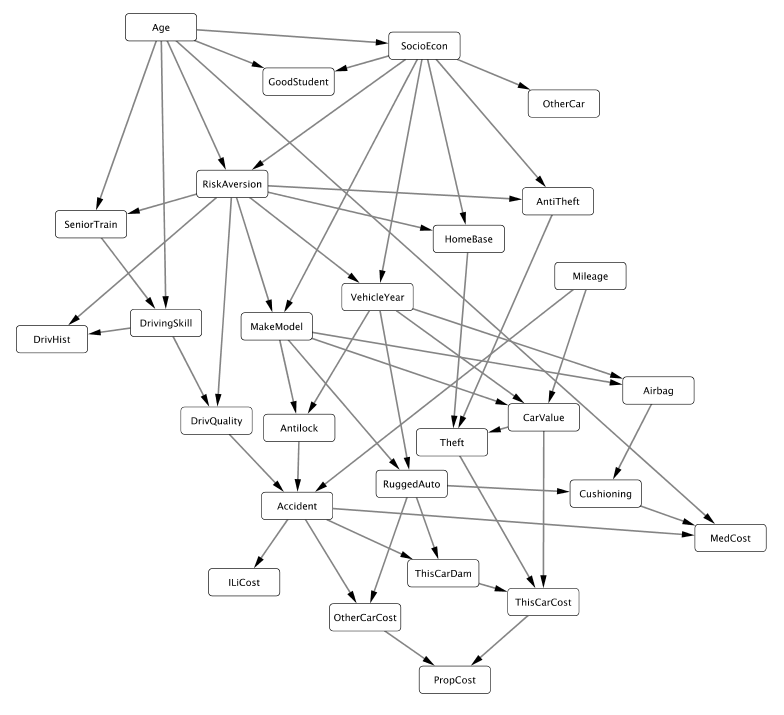}
 \caption{DAG structure of Insurance data. The intervention variable \textit{Accident} locates one level above of the outcome variable \textit{OtherCarCost}. And it can be seen that \textit{Accident} has three parent variables \{Mileage, Antilock, DrivQuality\}.}
 \label{fig:insurance}
\end{figure}

\begin{table*}[ht]
 \centering
 \resizebox{\textwidth}{!}{%
 \begin{tabular}{c|c|c}
 \multicolumn{1}{c|}{Adjustment variable candidates} & \multicolumn{1}{c|}{MI with intervention variable} & \multicolumn{1}{c}{Mean of cosine distance} \\
 \hline \hline
 \{RuggedAuto\} & 0.00016 & 0.00189 \\
 \{VehicleYear, MakeModel\} & 0.00354 & 0.00175 \\
 \{SocioEcon, RiskAversion, Antilock\} & 0.02178 & 0.00292 \\
 \{Age, RiskAversion, Antilock\} & 0.03186 & 0.00430 \\
 \{RiskAversion, DrivingSkill, Antilock\} & 0.26115 & 0.01358 \\
 \{Antilock, DrivQuality\} & 0.30299 & 0.02479 \\
 \multicolumn{1}{c}{ } \\
 \multicolumn{1}{c|}{Parents of intervention variable} & \multicolumn{1}{c|}{MI with intervention variable} & \multicolumn{1}{c}{Mean of cosine distance} \\
 \hline \hline
 \{Mileage, Antilock, DrivQuality\} & 0.40876 & 0.03262 \\
 \end{tabular}%
 }
 \caption{Mutual Information (MI) with the intervention variable and estimation error for each adjustment variable candidate in the Insurance data. This shows that when the MI with the intervention variable is smaller, the estimation error tends to be smaller too. And the parent set of the intervention variable \{Mileage, Antilock, DrivQuality\} has the largest amount of MI with the intervention variable and the largest estimation error.}
 \label{tab:table1}
\end{table*}

CAVS chooses \{RuggedAuto\} because it minimizes the mutual information with the intervention variables. As shown in Fig. \ref{fig:insurance}, \{RuggedAuto\} is located three hops from the intervention variable and \{VehicleYear, MakeModel\} is located two hops from the intervention variable. The results show that when the adjustment variable is further away from the intervention variable on the DAG, the amount of mutual information with the intervention variable is smaller and the estimation error is also smaller. Both \{Antilock, DrivQuality\} are parents of the intervention variable, and the mutual information with the intervention variable is the largest compared to the other candidate adjustment variables, and the estimation error is also the largest. This indicates that utilizing mutual information as one of the correlation indices can be an effective criterion for selecting adjustment variables. In addition, the parent set of the intervention variable \{Mileage, Antilock, DrivQuality\} that includes \{Mileage\} has the largest amount of mutual information with the intervention variable and the largest estimation error for the intervention calculation among all other sets of minimal variables that satisfy the back-door criteria. In particular, compared to \{Antilock, DrivQuality\}, the addition of the \{Mileage\} variable decreases the accuracy, indicating that it is not a good decision to simply use $pa(X)$ for adjustment variables for no other reason than that it is certain to satisfy the back-door criteria.

Next, in the Hailfinder example, Fig. \ref{fig:hailfinder} shows the overall structure of the DAG.
Let node \textit{R5Fcst} be the outcome variable, and \textit{N34StarFcst} be the intervention variable. We do not list all the candidate adjustment variables here, but there are 14 candidates in total, including the parent set of the intervention variable \{ScenRel3\_4, PlainsFcst\} and the last node in the back-door path from the intervention variable to the outcome variable \{MountainFcst\}. As in the Insurance example, the reference value of the outcome variable $Y$ was calculated using 200,000 samples, and the outcome variable $\hat{Y}$ was estimated using the subset data ${\mathbf D}_k$ with samples 2000, 1000, 500, 250, and 125. The error between $Y$ and $\hat{Y}$ was calculated using Eq. (\ref{error_cosd}).

\begin{figure}[t]
 \centering
 \includegraphics[width=1.0\columnwidth]{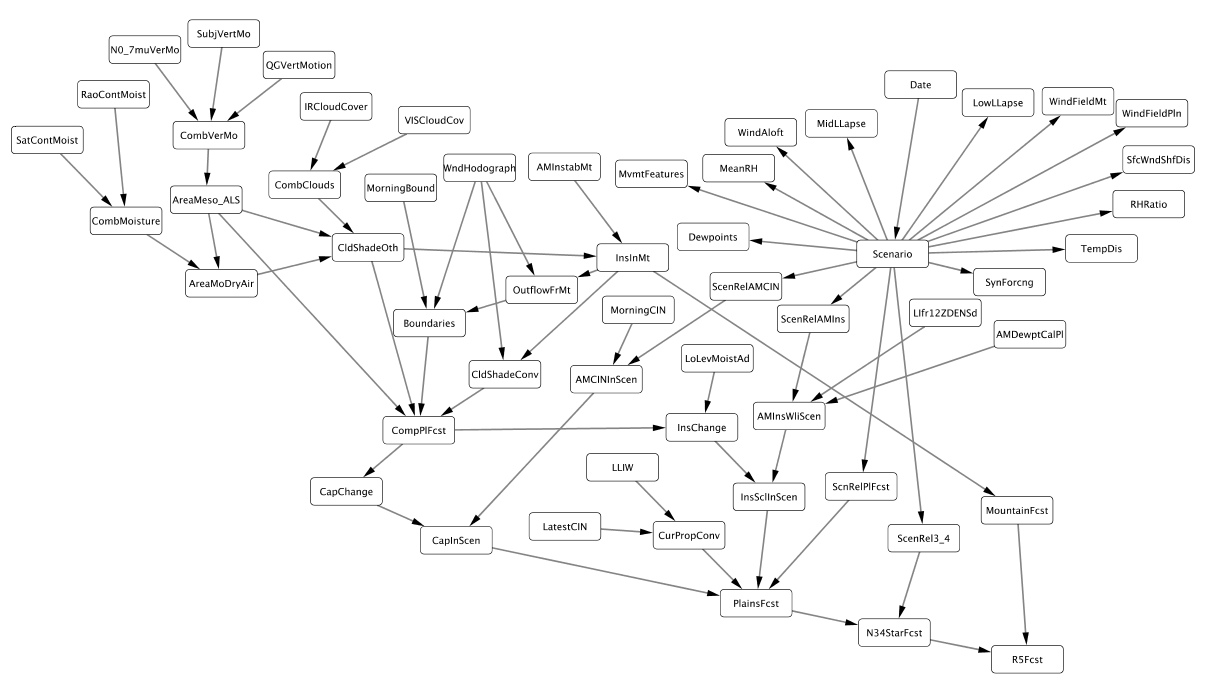}
 \caption{DAG structure of Hailfinder data. The intervention variable \textit{N34StarFcst} locates one level above of the outcome variable \textit{R5Fcst}. There are the parent set of the intervention variable \{ScenRel3\_4, PlainsFcst\} and the last node in the back-door path from the intervention variable to the outcome variable is \{MountainFcst\}.}
 \label{fig:hailfinder}
\end{figure}

\begin{figure}[t]
 \centering
 \includegraphics[width=0.8\columnwidth]{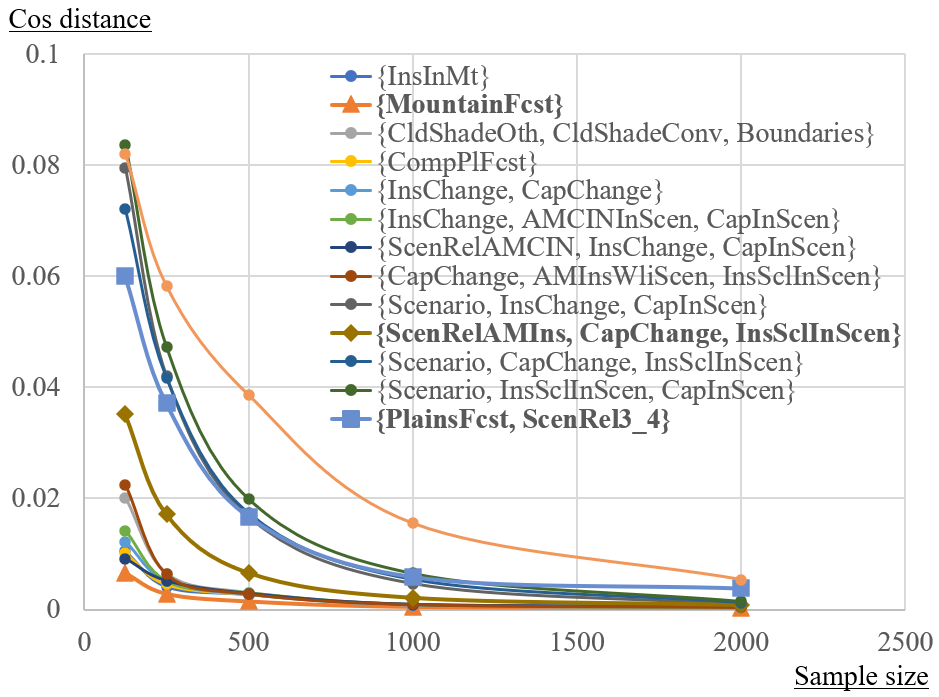}
 \caption{Relationship between sample size and error for each adjustment variable candidate in the Hailfinder data. There is an inversely proportional relationship between sample size and error. Particularly, the error increases sharply when the sample size is small. When \{MountainFcst\} is used as the adjustment variable, the error is the smallest of the whole, regardless of the sample size.}
 \label{fig:hailfinder_graph}
\end{figure}

The relationship between the number of samples and the estimation error for each candidate adjustment variable is shown in Fig. \ref{fig:hailfinder_graph}. It can be seen that the error increases significantly when the sample size decreases by half. The mutual information with the intervention variables was $7.6 \times 10^{-6}$ for \{MountainFcst\}, $0.061$ for \{Scenario, CapChange, InsSclInScen\}, and $0.68$ for \{ScenRel3\_4, PlainsFcst\}. In particular, when a variable having a small mutual information with $X$, such as \{MountainFcst\}, is used for adjustment, the error is not high even when the number of samples is small. This indicates that the adjustment variable having a small mutual information with X is robust to sample size.

\subsection{Artificial Data}
Next, we show the performance of CAVS on artificially generated data. The comparisons for CAVS are the smallest parents of $X$ that satisfy the back-door criterion, and $pa(X)$ (that certainly satisfy the back-door criterion). The data were generated as follows. DAGs were randomly generated with the number of nodes fixed at 30 and the number of edges fixed at 40. All variables were 4-valued categorical variables. For the conditional discrete probability distributions, the distributions among all variables were generated as independent uniform ones. We utilized the intervention results calculated using the 10,000 samples of data generated according to the conditional discrete probability distribution as the reference value for the outcome variable $Y$, and calculated the error with the outcome variable $\hat{Y}$ estimated using the subset data ${\mathbf D}_k$ obtained by randomly sampling 500 samples (2\% of the total). Here, instead of calculating the mean of the subset data ${\mathbf D}_k$ in Eq. (\ref{error_cosd}), we looked at the error for each subset data ${\mathbf D}_k$. We prepared four patterns of DAGs, three patterns of conditional discrete probability distributions for each DAG, and five patterns of random sampling, for a total of 60 ($4 \times 3 \times 5$) ways to calculate the error.

The results are shown in Fig. \ref{fig:artificial_experiment_result}. From left to right, a box-and-whisker plot of the errors for 60 samples is shown for when variables are selected by CAVS, when the smallest set of variables among $pa(X)$ that satisfy the back-door criterion are used, and when $pa(X)$ are used as adjustment variables. The mean of each error is $0.0077$, $0.0110$, and $0.0232$, respectively, indicating that CAVS reduces the error by about 30\% for the smallest set of variables in $pa(X)$ that satisfy the back-door criterion and by about 67\% for $pa(X)$. In addition, it is clear from Fig. \ref{fig:artificial_experiment_result} that CAVS has a smaller error variance than the two baselines, and that it provides good and stable results.

\begin{figure}[ht]
\centering
\includegraphics[width=10cm]
{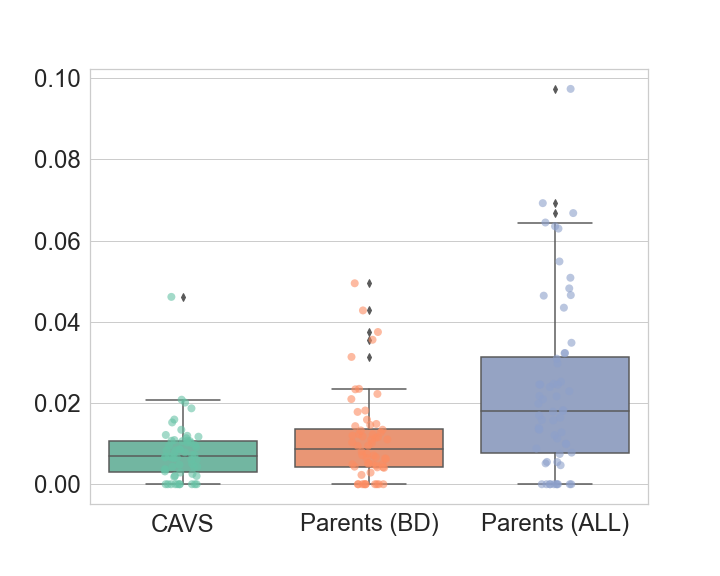}
\caption{Box plots of errors for CAVS and two baselines (the smallest parents of $X$ that satisfy the back-door criterion, parents of $X$) on artificial data. Of these, CAVS has the smallest error and the smallest error variance, indicating that it is the best and most stable result.}
\label{fig:artificial_experiment_result}
\end{figure}

\section{Conclusion}
We have proposed a method for selecting variables for covariate adjustment from both graph structure and data perspectives to maintain accuracy in causal effect estimation.
When there is not enough data, it is important to prioritize the selection of variables for this purpose.
To this end, we developed a two-step algorithm that consists of exploring candidates and selecting better variables using the correlations between related variables, and demonstrated its effectiveness in experiments with both existing and artificial data.
In particular, we showed that our method is robust to sample size and prevents accuracy degradation even with small amounts of data.
In addition, we demonstrated how to extend the application of the proposed method from DAGs to CPDAGs, and proved the theorem on computation possibility of interventions in CPDAGs. 
This broadens the applicability of intervention calculations in real-world data.

\section*{References}
\bibliography{iopart-num}

\end{document}